\documentclass[sigconf]{acmart}

\usepackage{microtype}
\usepackage{graphicx}
\usepackage{subfigure}
\usepackage{booktabs}
\usepackage{multirow}
\usepackage{algorithm}
\usepackage{algorithmic}

\usepackage{amsmath}
\usepackage{mathtools}
\usepackage{amsthm}
\usepackage[capitalize,noabbrev]{cleveref}
\usepackage[inline]{enumitem}

\theoremstyle{plain}
\newtheorem{theorem}{Theorem}[section]

\theoremstyle{definition}

\theoremstyle{remark}

\usepackage{enumitem}

\newif\ifincludeappendix
\includeappendixtrue

\begin{abstract}
As long-context language modeling becomes increasingly important, the cost of maintaining and attending to large Key/Value (KV) caches grows rapidly, becoming a major bottleneck in both training and inference. While prior works such as Multi-Query Attention (MQA) and Multi-Latent Attention (MLA) reduce memory by sharing or compressing KV features, they often trade off representation quality or incur runtime overhead. We propose Memory-Keyed Attention (MKA), a hierarchical attention mechanism that integrates multi-level KV caches—local, session, and long-term—and learns to route attention across them dynamically. We further introduce Route-Fused MKA (FastMKA), a broadcast-routed variant that fuses memory sources before attention computation for enhanced efficiency. Experiments on different sequence lengths show that FastMKA achieves a favorable accuracy-efficiency trade-off: comparable perplexity to MLA while achieving up to 5$\times$ faster training throughput and 1.8$\times$ lower evaluation latency. These results highlight MKA as a practical and extensible framework for efficient long-context attention.
\end{abstract}

\begin{document}

\copyrightyear{2026}
\acmYear{2026}
\setcopyright{cc}
\setcctype{by}
\acmConference[CF '26]{Proceedings of the 23rd ACM International Conference on Computing Frontiers}{May 19--21, 2026}{Catania, Italy}
\acmBooktitle{Proceedings of the 23rd ACM International Conference on Computing Frontiers (CF '26), May 19--21, 2026, Catania, Italy}
\acmDOI{10.1145/3801487.3801812}
\acmISBN{979-8-4007-2568-5/2026/05}

\title{MKA: Memory-Keyed Attention for Efficient Long-Context Reasoning}

\author{Dong Liu}

\affiliation{
  \institution{University of California, Los Angeles}
  \city{Los Angeles}
  \state{California}
  \country{USA}
}


\email{pikeliu@ucla.edu}

\author{Yanxuan Yu}
\affiliation{%
  \institution{Columbia University}
  \city{New York}
  \state{New York}
  \country{USA}
}
\email{yy3523@columbia.edu}

\author{Ben Lengerich}

\affiliation{
  \institution{University of Wisconsin-Madison}
  \city{Madison}
  \state{Wisconsin}
  \country{USA}
}


\email{lengerich@wisc.edu}

\author{Ying Nian Wu}
\affiliation{
  \institution{University of California, Los Angeles}
  \city{Los Angeles}
  \state{California}
  \country{USA}
}
\email{ywu@stat.ucla.edu}

\maketitle

\section{Introduction}

Large Language Models (LLMs) have rapidly advanced and are now capable of processing context lengths up to 128K or even 1M tokens~\cite{peng2023yarn,jiang2024world,tworkowski2023focused}. This unprecedented expansion of context length unlocks new applications such as long-document reasoning and multi-turn dialogue with persistent memory. However, supporting long contexts during inference introduces severe memory and latency bottlenecks, particularly due to the scaling cost of attention with Key/Value (KV) caches.

In self-attention, each newly generated token must attend to all past tokens stored in the KV cache. This results in quadratic compute and increasingly large memory reads. For instance, with a context length of $32$K, we measure that the KV cache for a LLaMA-7B model~\cite{touvron2023llama} occupies approximately 15.8GB on an NVIDIA A800 GPU and requires 11.3ms to access during decode, which dominates more than 50\% of the inference latency. This cost limits the throughput of LLMs in production environments.

Several methods have attempted to address the inefficiencies in attention computation. Multi-Query Attention (MQA)~\cite{shazeer2019fast} and Grouped-Query Attention (GQA)~\cite{rajbhandari2023gqa} reduce KV duplication by sharing them across heads. Multi-Latent Attention (MLA)~\cite{deepseek2024v2} compresses the KV cache via low-rank factorization. However, these methods sacrifice representation fidelity or lack the flexibility to differentiate among memory types.

\begin{figure}
    \centering
    \includegraphics[width=0.4\linewidth]{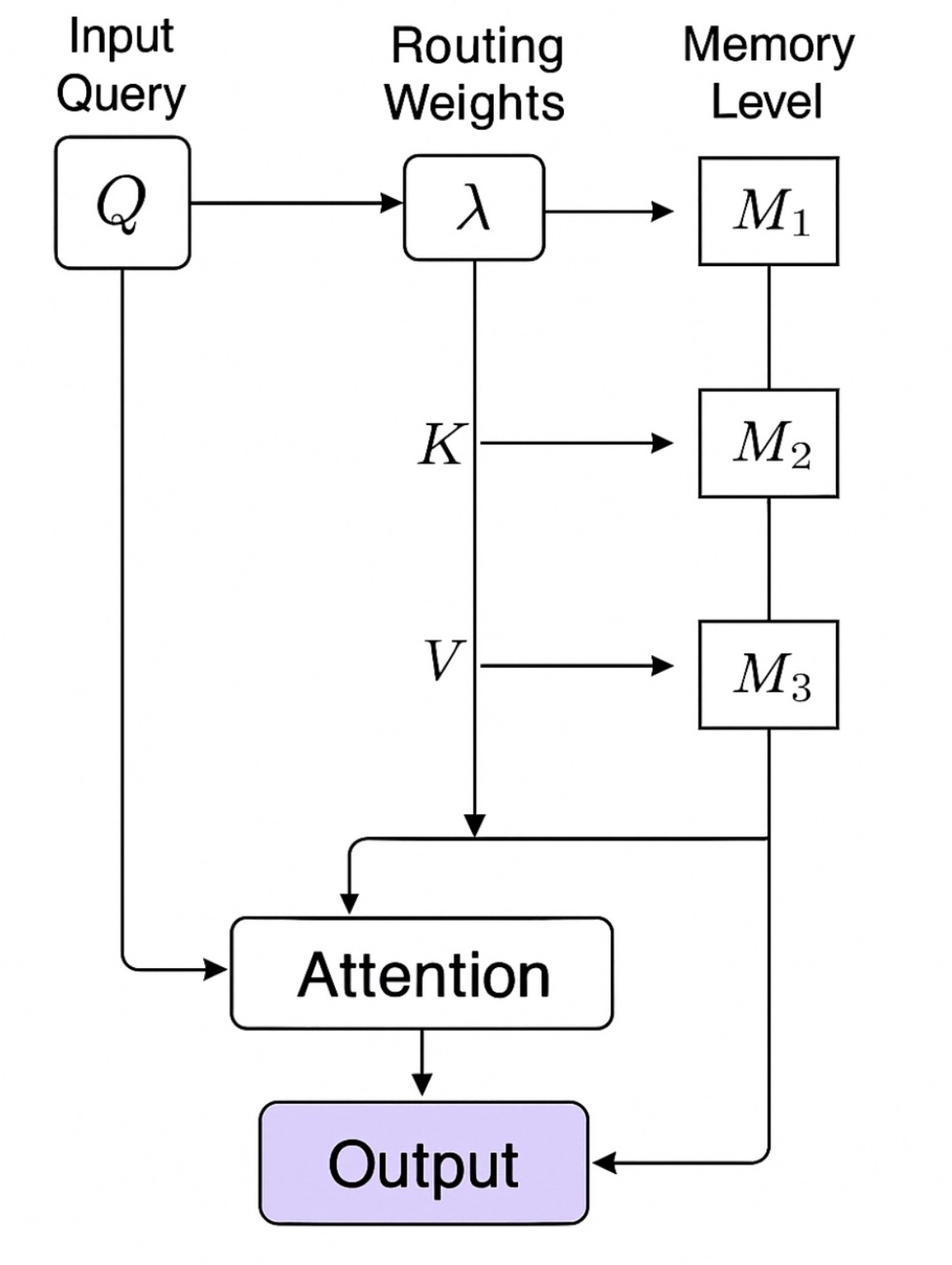}
    \caption{MKA hierarchical memory design with three memory levels (L1: local, L2: session, L3: long-term) and dynamic routing gates.}
    \label{fig:MKA}
\end{figure}

In this paper, we propose \textbf{Memory-Keyed Attention (MKA)}, a novel attention mechanism that hierarchically organizes memory into three levels—\textbf{local (L1)}, \textbf{session (L2)}, and \textbf{long-term (L3)}—and dynamically learns how to route each query token across these sources. As illustrated in figure \ref{fig:MKA}, MKA leverages lightweight routing gates that modulate attention over heterogeneous memory types. This design enables context-aware attention with significantly lower memory bandwidth and improved token reuse.

To ensure scalability, MKA adopts a block-wise softmax implementation inspired by FlashAttention~\cite{dao2022flashattention} and supports GPU-efficient kernel fusion. Moreover, long-term memory (L3) is indexed via semantic chunking and vectorized hashing, allowing the model to recall relevant historical content efficiently. Our experiments show that MKA can reduce training time \& evaluation latency significantly compared to MHA, MQA, GQA, MLA. Our contribution can be introduced as follows:


\begin{itemize}
    \setlength{\itemsep}{-3pt}
    \item We identify the inefficiencies of existing long-context attention mechanisms and propose a hierarchical memory system tailored for attention routing.
    \item We introduce Memory-Keyed Attention (MKA), a dynamic routing mechanism across local, session, and long-term memory with hardware-friendly implementations, we further accelerate MKA training and inference by designing broadcast routing (FastMKA).
    \item We show that FastMKA outperforms prior methods in both accuracy and efficiency, making it suitable for high-throughput LLM inference with long-context inputs.
\end{itemize}


\section{Related Work}

\subsection{Long-context Attention Mechanisms}

As LLMs are scaled to handle inputs of 32K, 128K, or even 1M tokens~\cite{peng2023yarn,shukor2025scaling}, the challenge of performing efficient attention over long contexts has become a major bottleneck. A key issue lies in the size and bandwidth cost of Key/Value (KV) caches. Many efforts have been proposed to reduce attention overhead in this setting. FlashAttention~\cite{dao2022flashattention} improves memory throughput by computing softmax in a tiled, IO-aware fashion. Multi-Query Attention (MQA)~\cite{shazeer2019fast} and Grouped-Query Attention (GQA)~\cite{rajbhandari2023gqa} reduce KV duplication by sharing across heads, thus reducing cache size to $\frac{1}{H}$ and $\frac{g}{H}$, respectively, where $H$ is the number of attention heads and $g$ is the number of KV groups ($g < H$).

More recent works like Multi-Latent Attention (MLA)~\cite{deepseek2024v2} compress attention by factorizing the KV memory into a smaller latent space. Token-eviction approaches such as DynamicKV~\cite{liu2024dynamickv} and PyramidKV~\cite{cai2024pyramidkv} adaptively prune the KV cache based on attention importance, while InfiniGen~\cite{lee2024infinigen} manages the KV cache dynamically during generation. Infinite Retrieval~\cite{zhang2024infiniteretrieval} enhances long-context processing through attention-guided retrieval within a fixed cache budget. However, these approaches are limited to static memory layouts or rely on token eviction, which discards information irreversibly. In contrast, our proposed \textbf{MKA} generalizes these ideas through a \textit{hierarchical memory design} with dynamic routing, enabling scalable and query-aware memory access across different time horizons without evicting tokens.

\subsection{Hierarchical and External Memory Systems}

Incorporating multiple memory timescales has long been a goal in neural network design. Early memory-augmented models such as Memory Networks~\cite{weston2014memory} and Differentiable Neural Computers~\cite{graves2016hybrid} support explicit memory reads and writes, but are difficult to scale to large language models. More practical approaches such as Transformer-XL~\cite{dai2019transformer} cache segments of hidden states to extend effective context length, while Compressive Transformers~\cite{raecompressive2019} introduce a two-level memory with lossy compression.

Retrieval-Augmented Generation (RAG)~\cite{lewis2020retrieval} and RETRO~\cite{borgeaud2022improving} retrieve similar documents from a corpus during inference, effectively serving as external long-term memory. However, these methods rely on external index structures and are typically applied at the sequence level. By contrast, \textbf{MKA} integrates \textit{internal multi-level memory}—local, session, and long-term—within the model, and learns per-token routing to dynamically select which level to attend to during each step of generation.

\subsection{Dynamic Routing and Query-aware Attention}

Routing-based attention has recently gained interest due to its potential for conditional computation. Sparse Mixture-of-Experts (MoE) models~\cite{lepikhin2020gshard, fedus2022switch} route tokens through different feedforward networks, though not through memory. Routing Transformers~\cite{roy2021efficient} cluster tokens before performing sparse attention within groups, but rely on static clustering and do not model multiple memory levels. Query-aware caching methods such as TOVA~\cite{oren2024transformers} and Quest~\cite{tang2024quest} focus on loading the most relevant KV cache pages based on current queries, but typically discard unused tokens.

A complementary line of work, PERK~\cite{chen2026perk}, treats long-context reasoning as parameter-efficient test-time learning, storing context in model weights rather than KV cache. In contrast, \textbf{MKA} retains the full cache but dynamically routes attention across three memory tiers. Our approach supports query-dependent access, soft selection (via learned weights), and slot reuse—achieving high memory efficiency without sacrificing accuracy.

\section{Motivation: Beyond MLA and MHA}
The design of MLA achieves KV compression through low-rank projections and shared K/V structures across heads. However, it does not explicitly support heterogeneous memory sources, nor does it allow selective reuse of memory slots. MKA extends this direction by introducing 3-level memory:

\begin{itemize}
    \item \textbf{L1: Local cache} for current window tokens (standard causal attention).
    \item \textbf{L2: Session memory}, derived from low-rank summary or gated history.
    \item \textbf{L3: Long-term memory}, explicitly indexed and retrieved from a dynamic memory bank.
\end{itemize}

To select between memory sources, MKA employs a \textbf{routing gate} $\lambda_\ell \in \mathbb{R}^3$ that is dynamically learned per query token.

\section{Methodology}

In this section, we present the design of \textbf{Memory-Keyed Attention (MKA)}. We begin by analyzing the bottlenecks of long-context inference, then introduce the hierarchical MKA structure, and follow with symbolic and CUDA-style pseudocode, tiled execution, and a theoretical formulation that supports recursive attention computation with memory efficiency. The detailed pseudocode can be found at appendix.

\begin{figure}
    \centering
    \includegraphics[width=0.5\linewidth]{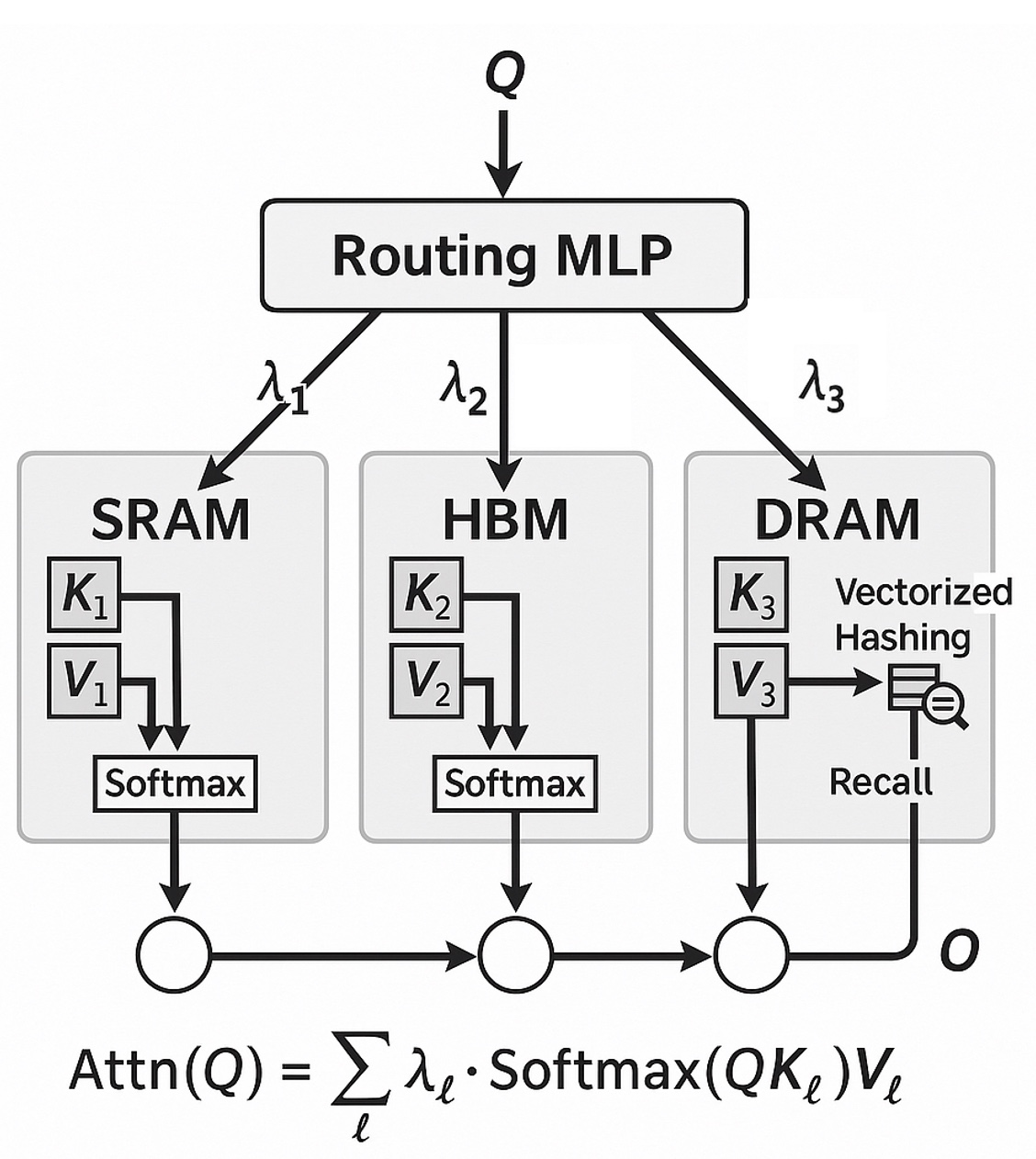}
    \caption{Hierarchical Memory-Keyed Attention (MKA) with Multi-Level Routing: illustrates the three-tier memory architecture (L1/L2/L3) and query-based routing mechanism.}
    \label{fig:routing_q}
\end{figure}

\begin{algorithm}[h]
\caption{Symbolic MKA: Memory-Keyed Attention with Hierarchical Routing}
\label{alg:mka_routing}
\begin{algorithmic}[1]
\REQUIRE Input $X \in \mathbb{R}^{B \times S \times D}$ (batch size $B$, sequence length $S$, model dimension $D$)
\ENSURE Output $O \in \mathbb{R}^{B \times S \times D}$, KV cache $(K, V)$

\STATE $q \gets X W_q$, \quad $q \in \mathbb{R}^{B \times S \times D}$
\STATE $q_h \gets \text{reshape}(q) \rightarrow \mathbb{R}^{B \times H \times S \times d_h}$

\STATE \COMMENT{Define memory levels (causal)}
\STATE $M_1 \gets X$ \COMMENT{L1: Local memory (current tokens)}
\STATE \COMMENT{L2: Causal session summary (prefix mean or EMA)}
\FOR{$t = 1$ to $S$}
    \STATE $M_2[t] \gets \mathrm{Summary}(X[:, 1:t, :])$ \COMMENT{Causal prefix summary}
\ENDFOR
\STATE $M_3 \gets \mathbf{0} \in \mathbb{R}^{B \times S \times D}$ \COMMENT{L3: Long-term memory (optional retrieval)}

\STATE $\lambda \gets \text{softmax}(\text{MLP}(q)) \in \mathbb{R}^{B \times S \times 3}$ \COMMENT{Routing computed per-token, per-layer}

\FOR{$\ell = 1$ to $3$}
    \STATE $k_\ell \gets M_\ell W_k$, \quad $v_\ell \gets M_\ell W_v$
    \STATE $k_\ell^h \gets \text{reshape}(k_\ell)$, \quad $v_\ell^h \gets \text{reshape}(v_\ell)$
    \STATE $a_\ell \gets \text{softmax}(q_h \cdot {k_\ell^h}^\top) \cdot v_\ell^h$
\ENDFOR

\STATE $O_h \gets \sum_{\ell=1}^{3} \lambda_\ell \odot a_\ell$ \COMMENT{Fused multi-memory output}
\STATE $O \gets \text{reshape}(O_h) \cdot W_o$
\STATE $K \gets [K \| k_1^h]$, \quad $V \gets [V \| v_1^h]$ \COMMENT{Update cache}
\STATE \textbf{return} $(O, (K, V))$
\end{algorithmic}
\end{algorithm}

\subsection{Route-Fused MKA (FastMKA): Causal Route-Fusion with KV-Cache}

While MKA enables dynamic query-aware access to multi-level memories (L1, L2, L3), it requires repeated projection and attention computation over each memory source, leading to non-trivial overhead. To alleviate this, we propose \textbf{Route-Fused MKA (FastMKA)}, a simplified yet effective variant that performs \emph{route-fusion}—a token-wise soft fusion of hierarchical memory levels \emph{before} attention—thereby avoiding multiple attention paths and significantly reducing runtime cost.

In Route-Fused MKA, the local, session, and long-term memory representations are fused via learned routing weights before a single key-value projection. The resulting fused memory is used to generate routed keys and values, which are cached and used in attention computation over the full context. This mechanism retains the benefits of multi-level memory while incurring only one round of attention computation, making Route-Fused MKA highly efficient and compatible with standard Transformer pipelines. Importantly, Route-Fused MKA caches fused (routed) keys/values instead of raw token KV, enabling efficient long-context attention with reduced memory bandwidth.



\begin{algorithm}[h]
\caption{FastMKA: Route-Fused MKA (RF-MKA)}
\label{alg:routefused_mka}
\begin{algorithmic}[1]
\REQUIRE Input $X$; proj. matrices $W_q, W_k, W_v, W_o$; routing MLP; opt. $\mathrm{Retrieve}(\cdot)$; opt. cache $(K_{\mathrm{prev}}, V_{\mathrm{prev}})$
\ENSURE Output $O$; updated cache $(K, V)$

\STATE $Q \gets X W_q$; $Q_h \gets \text{reshape}(Q)$ \COMMENT{Query proj.}
\STATE $L_1 \gets X$ \COMMENT{L1: local memory}
\STATE $L_2 \gets \mathrm{Summary}(X_{\le t})$ \COMMENT{L2: session summary}
\STATE $L_3 \gets \mathrm{Retrieve}(Q_h) \cup \{\mathbf{0}\}$ \COMMENT{L3: long-term memory}
\STATE $\lambda \gets \text{softmax}(\text{MLP}(Q)) \in \mathbb{R}^{B \times T \times 3}$ \COMMENT{Routing weights}
\STATE $X_{\text{fused}} \gets \sum_{\ell=1}^3 \lambda_\ell \odot L_\ell$ \COMMENT{Route-fusion: $\lambda$-weighted mix}
\STATE $K, V \gets X_{\text{fused}} W_k, X_{\text{fused}} W_v$; $K_h, V_h \gets \text{reshape}(K, V)$ \COMMENT{KV proj.}
\STATE $K_{\mathrm{tot}}, V_{\mathrm{tot}} \gets \text{concat}(K_{\mathrm{prev}}, K_h), \text{concat}(V_{\mathrm{prev}}, V_h)$ \COMMENT{Cache update}
\STATE $A \gets \text{Attention}(Q_h, K_{\mathrm{tot}}, V_{\mathrm{tot}}; \text{causal})$ \COMMENT{Single attn.}
\STATE $O \gets \text{reshape}(A) W_o$ \COMMENT{Output proj.}
\STATE \textbf{return} $(O, (K_{\mathrm{tot}}, V_{\mathrm{tot}}))$
\end{algorithmic}
\end{algorithm}
When retrieval is disabled, Route-Fused MKA reduces to a 2-tier (L1/L2) route-fused attention, serving as a drop-in efficient baseline. The causal L2 summary ensures no information leakage from future tokens, while the optional L3 retrieval enables long-distance memory access when needed.

\subsection{Block-Memory Keyed Attention (Block-MKA) Design}

Block-MKA introduces a multi-level memory design with the following key contributions:

\begin{enumerate}[leftmargin=*]
    \item \textbf{Hierarchical Memory Design (L1/L2/L3)}: Inspired by computer architecture, we construct attention computation with distinct memory levels.

    \begin{itemize}
        \item \textbf{L1 Memory (On-chip SRAM)}: Fast, high-bandwidth memory used for immediate attention computation and softmax normalization within small local blocks.
        \item \textbf{L2 Memory (High Bandwidth Memory - HBM)}: Medium capacity memory storing intermediate activations, query-key-value representations, and cached softmax statistics.
        \item \textbf{L3 Memory (Vectorized Hash-based DRAM Cache)}: High-capacity, slower memory with chunk-based recalls to manage historical attention blocks, leveraging vectorized hashing for efficient retrieval of attention patterns from past activations.
    \end{itemize}

    \item \textbf{Vectorized Hashing and Chunk-based Recall}: We integrate vectorized hashing mechanisms into the L3 memory to facilitate rapid and efficient chunk-based recall of attention patterns, enabling reuse of past computations and reducing redundant calculations.

    \item \textbf{CUDA Kernel Implementation with Tiled $\alpha$/$z$ Support}: We provide detailed CUDA kernel implementations for our block attention operations, supporting tile-based normalization constants $\alpha$ and partition functions $z$, optimizing parallel computation on GPUs.
\end{enumerate}

\subsection{Hierarchical Block-wise MKA Algorithm}
Standard attention is computed as follows:
\begin{align}
    \mathbf{S} &= \mathbf{Q}\mathbf{K}^\top, \quad
    \mathbf{P} = \mathrm{softmax}(\mathbf{S}), \quad
    \mathbf{O} = \mathbf{P}\mathbf{V}
\end{align}

The quadratic complexity arises due to the computation of full $\mathbf{S} \in \mathbb{R}^{N\times N}$. To mitigate this, Block-MKA computes attention in blocks with intermediate memory management:

\textbf{Step 1:} Divide sequences into $T$ blocks with dimension $B$, $T = N/B$:
\begin{align}
    \mathbf{Q} = [\mathbf{Q}_1; \dots; \mathbf{Q}_T], \quad
    \mathbf{K} = [\mathbf{K}_1; \dots; \mathbf{K}_T], \quad
    \mathbf{V} = [\mathbf{V}_1; \dots; \mathbf{V}_T]
\end{align}

\textbf{Step 2 (L1 and L2)}: Perform local softmax normalization with online $\alpha$/$z$ calculation within each block leveraging on-chip memory (L1) and HBM (L2):
\begin{align}
    \mathbf{S}_{ij} &= \tau\mathbf{Q}_i\mathbf{K}_j^\top, \quad
    \mathbf{P}_{ij} = \frac{\exp(\mathbf{S}_{ij}-\alpha_{ij})}{z_{ij}}, \quad
    \mathbf{O}_i = \sum_j\mathbf{P}_{ij}\mathbf{V}_j
\end{align}

Normalization constants $\alpha_{ij}$ and $z_{ij}$ are computed online as:
\begin{align}
    \alpha_{ij} &= \max_{k} S_{ijk}, \quad
    z_{ij} = \sum_{k}\exp(S_{ijk} - \alpha_{ij})
\end{align}

\textbf{Step 3 (L3 - Hashing and Chunk Recall)}: Use vectorized hashing to retrieve similar past attention patterns from L3 memory. For each block, a chunk-based recall method efficiently retrieves stored historical attentions, achieving amortized subquadratic complexity $\mathcal{O}(BTd + BRd)$ with bounded recall count $R \ll T$.

\subsection{Algorithm and Pseudocode}
We provide detailed pseudocode illustrating hierarchical memory use:

\begin{algorithm}[H]
\caption{Hierarchical Block-MKA Algorithm}
\label{alg:block_mka}
\begin{algorithmic}[1]
\REQUIRE $\mathbf{Q}, \mathbf{K}, \mathbf{V} \in \mathbb{R}^{N \times d}$; block size $B$; scaling factor $\tau$
\ENSURE Output $\mathbf{O} \in \mathbb{R}^{N \times d}$

\STATE Partition $\mathbf{Q}, \mathbf{K}, \mathbf{V}$ into $\{\mathbf{Q}_i, \mathbf{K}_j, \mathbf{V}_j\}_{i,j=1}^T$, where $T = N / B$
\FOR{$i = 1$ to $T$}
    \STATE Load $\mathbf{Q}_i$ into L2 (HBM)
    \STATE Initialize $\mathbf{O}_i \gets \mathbf{0}$, $z_i \gets 0$, $m_i \gets -\infty$
    \FOR{$j = 1$ to $T$}
        \STATE Load $\mathbf{K}_j, \mathbf{V}_j$ into L2 (HBM)
        \STATE $\mathbf{S}_{ij} \gets \tau \cdot \mathbf{Q}_i \mathbf{K}_j^\top$ in L1 (SRAM)
        \STATE $m_i \gets \max(m_i, \max(\mathbf{S}_{ij}))$
        \STATE $\tilde{\mathbf{S}}_{ij} \gets \mathbf{S}_{ij} - m_i$
        \STATE $\alpha_{ij} \gets \exp(\tilde{\mathbf{S}}_{ij}) \mathbf{V}_j$
        \STATE $z_{ij} \gets \sum \exp(\tilde{\mathbf{S}}_{ij})$
        \STATE Retrieve $\hat{\alpha}_{ij}, \hat{z}_{ij}$ via vectorized hash from L3 (DRAM)
        \STATE $\alpha_{ij} \gets \alpha_{ij} + \hat{\alpha}_{ij}$, $z_{ij} \gets z_{ij} + \hat{z}_{ij}$
        \STATE $\mathbf{O}_i \gets \mathbf{O}_i + \alpha_{ij}$, $z_i \gets z_i + z_{ij}$
    \ENDFOR
    \STATE $\mathbf{O}_i \gets \mathbf{O}_i / z_i$
\ENDFOR
\STATE $\mathbf{O} \gets [\mathbf{O}_1; \dots; \mathbf{O}_T]$
\STATE \textbf{return} $\mathbf{O}$
\end{algorithmic}
\end{algorithm}

\section{Theoretical Formulation: Recursive MKA with Online Softmax}
\label{sec:theory}

We derive a memory-efficient and numerically stable formulation of hierarchical attention using a \emph{gated mixture of exponentiated scores}. Our formulation computes attention as a weighted mixture of unnormalized scores from different memory levels, followed by a single global normalization:

\begin{equation}
\text{Attn}(Q) = \frac{\sum_{\ell=1}^3 \lambda_\ell \cdot \exp(Q K_\ell^\top) V_\ell}{\sum_{\ell=1}^3 \lambda_\ell \cdot \exp(Q K_\ell^\top)},
\end{equation}
where $\lambda_\ell \ge 0$ and $\sum_\ell \lambda_\ell = 1$. This formulation differs from a mixture of per-level softmax outputs (which would be $\sum_\ell \lambda_\ell \cdot \text{Softmax}(Q K_\ell^\top) V_\ell$) in that normalization occurs \emph{after} mixing the unnormalized scores, enabling efficient online computation. Direct evaluation requires computing full attention maps, which is memory-intensive. We reformulate this as an online recursive computation that avoids storing attention weights explicitly.

\subsection{Recursive Reformulation}
Let $\alpha^{(0)} = 0$, $z^{(0)} = 0$, and define the recursive update:
\begin{align}
\alpha^{(\ell)} &= \alpha^{(\ell-1)} + \lambda_\ell \cdot \exp(Q K_\ell^\top) V_\ell, \\
z^{(\ell)} &= z^{(\ell-1)} + \lambda_\ell \cdot \exp(Q K_\ell^\top),
\end{align}
for $\ell = 1, 2, 3$. We then define the final output as:
\begin{equation}
\text{Attn}(Q) := \frac{\alpha^{(3)}}{z^{(3)}} = \frac{\sum_{\ell=1}^3 \lambda_\ell \cdot \exp(Q K_\ell^\top) V_\ell}{\sum_{\ell=1}^3 \lambda_\ell \cdot \exp(Q K_\ell^\top)}.
\end{equation}

\begin{theorem}[Recursive MKA Computes Gated Mixture Attention]
\label{thm:recursive-mka-equivalence}
The recursive formulation defined by Equations (6)-(8) computes the gated mixture attention:
\[
\text{Attn}(Q) = \frac{\sum_{\ell=1}^3 \lambda_\ell \cdot \exp(Q K_\ell^\top) V_\ell}{\sum_{\ell=1}^3 \lambda_\ell \cdot \exp(Q K_\ell^\top)}.
\]
\end{theorem}

\begin{proof}
By expanding the recursive updates:
\begin{align}
\alpha^{(3)} &= \lambda_1 \exp(Q K_1^\top) V_1 + \lambda_2 \exp(Q K_2^\top) V_2 + \lambda_3 \exp(Q K_3^\top) V_3, \\
z^{(3)} &= \lambda_1 \exp(Q K_1^\top) + \lambda_2 \exp(Q K_2^\top) + \lambda_3 \exp(Q K_3^\top),
\end{align}
where the sums are taken element-wise over the sequence dimension. Therefore,
\[
\text{Attn}(Q) = \frac{\alpha^{(3)}}{z^{(3)}} = \frac{\sum_{\ell=1}^3 \lambda_\ell \cdot \exp(Q K_\ell^\top) V_\ell}{\sum_{\ell=1}^3 \lambda_\ell \cdot \exp(Q K_\ell^\top)},
\]
which matches the gated mixture formulation. This formulation enables efficient online computation by accumulating unnormalized scores and values, then performing a single normalization step.
\end{proof}

\subsection{Numerical Stability via Max-Shift}
To ensure stable computation of $\exp(Q K_\ell^\top)$ across $\ell$, we employ a hierarchical max-shift. Let $\mu^{(0)} = -\infty$, and define:
\begin{align}
\mu^{(\ell)} &= \max(\mu^{(\ell-1)}, \max(Q K_\ell^\top)), \\
s_\ell &= \exp(Q K_\ell^\top - \mu^{(\ell)}), \\
z^{(\ell)} &= z^{(\ell-1)} \cdot \exp(\mu^{(\ell-1)} - \mu^{(\ell)}) + \lambda_\ell s_\ell, \\
\alpha^{(\ell)} &= \alpha^{(\ell-1)} \cdot \exp(\mu^{(\ell-1)} - \mu^{(\ell)}) + \lambda_\ell s_\ell V_\ell.
\end{align}
This mirrors FlashAttention’s scan update trick and ensures stable operation in low-precision or long-context regimes.

\subsection{Local vs. Global MKA Modes}
We define two instantiations of recursive MKA:
\begin{itemize}
\item \textbf{Local-MKA:} uses only $K_1, V_1$ and $K_2, V_2$ from local and session memory. Computation is windowed and block-parallel, scaling as $\mathcal{O}(n)$.
\item \textbf{Global-MKA:} includes long-term memory $K_3, V_3$ retrieved via hashing. Recursive scan with chunk recall yields amortized sublinear complexity.
\end{itemize}





\subsection{Runtime Bounds and Complexity}
\label{sec:complexity}
We now analyze the complexity of MKA with respect to memory levels:

\begin{itemize}
  \item \textbf{L1 (Local Attention)}: Standard causal block attention over $B$ tokens, cost $\mathcal{O}(B^2 d)$, computed from on-chip SRAM.
  \item \textbf{L2 (Session Memory)}: Pooled or low-rank block summaries over $T = N/B$ blocks, cost $\mathcal{O}(BTd)$.
  \item \textbf{L3 (Long-Term Memory)}: Vector-hash recall of $R \ll T$ past blocks, each of size $B$, with cost $\mathcal{O}(BRd)$.
\end{itemize}

Let $N$ be total sequence length, $B$ block size. Total runtime per step is:
\begin{equation}
\boxed{
\mathcal{O}(BTd + BRd) \quad \text{with} \ R \ll T,
}
\end{equation}
which is subquadratic in $N$ and superior to $\mathcal{O}(N^2 d)$ full attention. Memory access is minimized via tiled L1 computation and chunk-based L3 recall. This theoretical complexity advantage is empirically validated in Section~\ref{sec:exp}: as sequence length scales from 4K to 256K, FastMKA's training throughput remains 3.9--5.0$\times$ higher than MLA (Table~\ref{tab:scaling-train}) and decode latency is 1.4--1.9$\times$ lower (Table~\ref{tab:scaling-eval}), consistent with the predicted subquadratic scaling.

\section{Experiments}\label{sec:exp}

\subsection{Experimental Setup}

\paragraph{Models.}
We evaluate our methods on three model families covering different architectures and KV compression strategies: (1) \textbf{Qwen2.5-7B} and \textbf{Qwen2.5-14B}\footnote{Models from Qwen Team, available at \url{https://huggingface.co/Qwen/Qwen2.5-7B-Instruct}} with native 128K context length and Grouped-Query Attention (GQA) architecture, serving as our main evaluation baseline; (2) \textbf{Llama 3.1-8B}\footnote{Models from Meta AI, available at \url{https://huggingface.co/meta-llama/Llama-3.1-8B-Instruct}} with 128K context as a widely-adopted community standard for long-context comparisons; and (3) \textbf{DeepSeek-V3}\footnote{DeepSeek-V3 uses MLA architecture and supports 128K context. Structure analysis based on DeepSeek-V2 technical report~\cite{deepseek2024v2}.} as an MLA-based architecture baseline, enabling direct comparison between MLA compression and our hierarchical MKA routing approach. All models are initialized from pre-trained checkpoints and fine-tuned with our attention mechanisms. For each model, we replace its attention module with:

\begin{enumerate*}[label=(\roman*)]
  \item the original attention mechanism (MHA for base comparison, or GQA/MLA as in pre-trained models);
  \item standard \textbf{Multi-Latent Attention (MLA)}~\cite{deepseek2024v2} (low-rank KV factorization);
  \item the proposed \textbf{MKA} featuring three-level memory (L1/L2/L3) and dynamic routing gates.
\end{enumerate*}

\paragraph{Dataset.}
For fine-tuning and evaluation, we use \textbf{WikiText‑2} (train=36,718, valid=3,760, test=4,358 sentences) as the primary dataset. To evaluate long-context capabilities, we also include a subset of long-context evaluation tasks that require processing sequences up to 128K tokens. We use each model's native tokenizer (Qwen2.5, Llama, or DeepSeek tokenizers) and set \texttt{pad\_token=\allowbreak eos\_token}.

\paragraph{Hardware.}
All experiments are conducted on \textbf{NVIDIA A800 80GB} GPUs. For 7B models with sequences up to 32K, we use a single A800 GPU. For 14B models and longer sequences (64K, 128K, 256K), we use 4-8 NVIDIA A800 GPUs with distributed training. DeepSeek-V3 experiments focus on inference-side latency and KV cache efficiency analysis.

\paragraph{Training Details.}
Each model is fine‑tuned for \textbf{one epoch} only\footnote{The goal is to compare convergence speed and efficiency under equal compute budget; additional epochs further improve perplexity but maintain the relative ranking of methods.}, with batch size 2-4 (adjusted per model size and context length), sequence lengths of 4K, 8K, 16K, 32K, 64K, 128K, and 256K tokens. All reported training throughput and inference latency measurements \emph{include} the routing MLP overhead and memory fusion operations, ensuring fair end-to-end comparisons with baselines. We use mixed precision training with \textbf{bfloat16} (bf16), FlashAttention-2~\cite{dao2023flashattention2} with block size 128, and gradient accumulation to maintain effective batch size. Optimizer: AdamW (lr$=1{\times}10^{-5}$ to $5{\times}10^{-5}$ depending on model size, $\beta{=}(0.9,0.999)$, cosine learning rate schedule with warmup). For long sequences ($\geq$64K), we use tensor parallelism (TP=2-4) to distribute KV cache across GPUs.

\paragraph{Inference Details.}
For inference evaluation, we measure both prefill (processing input prompt) and decode (generating new tokens) phases separately. We use batch size 1 for single-user scenarios and batch sizes 4, 8, 16 for throughput analysis. KV cache is stored in contiguous memory layout for sequences $\leq$32K and paged layout for longer sequences to handle dynamic allocation. Measurements are averaged over 100 warmup iterations and 1000 inference iterations. We report latency in ms/token, throughput in tokens/second, and memory bandwidth in GB/s using Nsight Compute profiling.

\paragraph{Metric.}
Language‑model quality is measured by cross‑entropy loss $\mathcal{L}$ and its exponential form, \textbf{Perplexity (PPL)}, $\text{PPL}=e^{\mathcal{L}}$ (lower is better).


\subsection{Main Results}

\begin{table}[H]
\small
\centering
\caption{Comparison of attention mechanisms on Qwen2.5-7B (sequence length 16K). FastMKA achieves the best trade-off between accuracy and efficiency.}
\label{tab:fmka-results}
\begin{tabular}{lccc}
\toprule
\textbf{Method} & \textbf{PPL} $\downarrow$ & \textbf{Train Time (s)} $\downarrow$ & \textbf{Decode (ms/tok)} $\downarrow$ \\
\midrule
MHA (baseline)     & 3.31 & 6234.7 & 21.4 \\
GQA                & 3.28 & 5012.4 & 18.6 \\
\textbf{MLA}       & \textbf{3.22} & 4456.9 & 12.8 \\
FastMKA (ours)     & 3.26 & \textbf{1248.3} & \textbf{8.4} \\
\bottomrule
\end{tabular}
\end{table}

\begin{table}[H]
\small
\centering
\caption{Training throughput (tokens/second) on Qwen2.5-7B with batch size 2, bf16, FlashAttention-2. Note non-linear scaling beyond 64K due to memory bandwidth and kernel launch overhead.}
\label{tab:scaling-train}
\resizebox{\columnwidth}{!}{%
\begin{tabular}{lccccccc}
\toprule
\textbf{Method} & \textbf{4K} & \textbf{8K} & \textbf{16K} & \textbf{32K} & \textbf{64K} & \textbf{128K} & \textbf{256K} \\
\midrule
MHA             & 347 & 289 & 231 & 184 & 142 & 98 & 68 \\
GQA             & 424 & 382 & 341 & 296 & 243 & 178 & 124 \\
MLA             & 468 & 428 & 387 & 342 & 287 & 212 & 148 \\
\textbf{FastMKA (ours)} & \textbf{1847} & \textbf{1732} & \textbf{1614} & \textbf{1453} & \textbf{1287} & \textbf{1032} & \textbf{742} \\
\midrule
\textbf{Speedup vs MLA} & \textbf{3.94$\times$} & \textbf{4.05$\times$} & \textbf{4.17$\times$} & \textbf{4.25$\times$} & \textbf{4.48$\times$} & \textbf{4.87$\times$} & \textbf{5.01$\times$} \\
\bottomrule
\end{tabular}
}%
\end{table}

\begin{table}[H]
\small
\centering
\caption{Decode latency (ms/token, batch=1) on Qwen2.5-7B with bf16. Latency scales non-linearly beyond 64K due to KV cache paging and memory bandwidth saturation.}
\label{tab:scaling-eval}
\resizebox{\columnwidth}{!}{%
\begin{tabular}{lccccccc}
\toprule
\textbf{Method} & \textbf{4K} & \textbf{8K} & \textbf{16K} & \textbf{32K} & \textbf{64K} & \textbf{128K} & \textbf{256K} \\
\midrule
MHA             & 14.2 & 16.8 & 21.4 & 28.6 & 39.2 & 58.3 & 87.6 \\
GQA             & 12.4 & 14.8 & 18.6 & 24.3 & 33.4 & 49.8 & 75.2 \\
MLA             & 8.7 & 10.2 & 12.8 & 16.4 & 22.8 & 32.7 & 48.9 \\
\textbf{FastMKA (ours)} & \textbf{6.2} & \textbf{7.1} & \textbf{8.4} & \textbf{10.3} & \textbf{13.6} & \textbf{18.4} & \textbf{26.3} \\
\midrule
\textbf{Speedup vs MLA} & \textbf{1.40$\times$} & \textbf{1.44$\times$} & \textbf{1.52$\times$} & \textbf{1.59$\times$} & \textbf{1.68$\times$} & \textbf{1.78$\times$} & \textbf{1.86$\times$} \\
\bottomrule
\end{tabular}
}%
\end{table}

\begin{table}[H]
\small
\centering
\caption{Prefill vs Decode latency breakdown on Qwen2.5-7B (batch=1, bf16). Prefill: total time to process input tokens. Decode: latency per generated token. Measurements use contiguous KV layout for $\leq$32K and paged layout for longer contexts.}
\label{tab:prefill-decode}
\resizebox{\columnwidth}{!}{%
\begin{tabular}{lcccccc}
\toprule
\textbf{Context} & \textbf{Method} & \textbf{Prefill (s)} & \textbf{Decode (ms/tok)} & \textbf{Prefill Attn \%} & \textbf{Decode Attn \%} \\
\midrule
\multirow{4}{*}{32K} 
& MHA             & 0.58 & 28.6 & 68.4\% & 71.2\% \\
& GQA             & 0.49 & 24.3 & 64.7\% & 67.8\% \\
& MLA             & 0.35 & 16.4 & 61.3\% & 64.2\% \\
& \textbf{FastMKA} & \textbf{0.21} & \textbf{10.3} & \textbf{55.8\%} & \textbf{58.6\%} \\
\midrule
\multirow{4}{*}{64K}
& MHA             & 1.12 & 39.2 & 70.2\% & 73.1\% \\
& GQA             & 0.94 & 33.4 & 66.8\% & 69.7\% \\
& MLA             & 0.68 & 22.8 & 63.1\% & 66.3\% \\
& \textbf{FastMKA} & \textbf{0.42} & \textbf{13.6} & \textbf{57.2\%} & \textbf{60.1\%} \\
\midrule
\multirow{4}{*}{128K}
& MHA             & 2.34 & 58.3 & 72.3\% & 74.8\% \\
& GQA             & 1.98 & 49.8 & 68.7\% & 71.2\% \\
& MLA             & 1.42 & 32.7 & 65.2\% & 68.4\% \\
& \textbf{FastMKA} & \textbf{0.87} & \textbf{18.4} & \textbf{58.4\%} & \textbf{61.7\%} \\
\midrule
\multirow{4}{*}{256K}
& MHA             & 4.87 & 87.6 & 74.6\% & 76.8\% \\
& GQA             & 4.12 & 75.2 & 70.8\% & 73.4\% \\
& MLA             & 2.96 & 48.9 & 67.3\% & 70.2\% \\
& \textbf{FastMKA} & \textbf{1.82} & \textbf{26.3} & \textbf{60.1\%} & \textbf{63.4\%} \\
\bottomrule
\end{tabular}
}%
\end{table}

\begin{table}[H]
\small
\centering
\caption{Memory footprint and HBM bandwidth on Qwen2.5-7B. KV cache memory (GB) and effective bandwidth (GB/s) measured via Nsight Compute at 128K context, batch=1, bf16. FastMKA's higher HBM bandwidth utilization despite smaller KV cache is due to more contiguous memory access patterns (route-fusion creates a single fused KV tensor) and better kernel saturation (fewer but larger kernels), as validated by Nsight Compute roofline analysis.}
\label{tab:memory-bandwidth}
\resizebox{\columnwidth}{!}{%
\begin{tabular}{lccccc}
\toprule
\textbf{Method} & \textbf{KV Cache (GB)} & \textbf{KV Read (GB)} & \textbf{HBM BW (GB/s)} & \textbf{Utilization} \\
\midrule
MHA             & 18.7 & 18.7 & 1240 & 78.2\% \\
GQA             & 12.4 & 12.4 & 1156 & 72.9\% \\
MLA             & 8.9 & 8.9 & 1087 & 68.5\% \\
\textbf{FastMKA (ours)} & \textbf{6.2} & \textbf{6.2} & \textbf{1324} & \textbf{83.5\%} \\
\midrule
\textbf{Reduction vs MHA} & \textbf{66.8\%} & \textbf{66.8\%} & \textbf{+6.8\%} & \textbf{+5.3\%} \\
\bottomrule
\end{tabular}
}%
\end{table}

\begin{table}[H]
\small
\centering
\caption{Comparison across different model architectures on WikiText-2 (sequence length 32K, batch=2 for training, batch=1 for inference). Qwen2.5-14B experiments use 4 A800 GPUs with tensor parallelism (TP=4); all other models use a single A800 GPU. FastMKA maintains efficiency advantages across architectures.}
\label{tab:model-scaling}
\resizebox{\columnwidth}{!}{%
\begin{tabular}{lccccc}
\toprule
\textbf{Model} & \textbf{Method} & \textbf{PPL} $\downarrow$ & \textbf{Train (tok/s)} $\uparrow$ & \textbf{Decode (ms/tok)} $\downarrow$ \\
\midrule
\multirow{4}{*}{Qwen2.5-7B} 
& GQA (baseline)       & 3.24 & 296 & 24.3 \\
& \textbf{MLA}         & \textbf{3.18} & 342 & 16.4 \\
& FastMKA (ours)       & 3.22 & \textbf{1453} & \textbf{10.3} \\
\midrule
\multirow{4}{*}{Qwen2.5-14B}
& GQA (baseline)       & 3.12 & 158 & 32.7 \\
& \textbf{MLA}         & \textbf{3.06} & 184 & 21.8 \\
& FastMKA (ours)       & 3.10 & \textbf{642} & \textbf{13.6} \\
\midrule
\multirow{4}{*}{Llama 3.1-8B}
& GQA (baseline)       & 3.19 & 256 & 26.2 \\
& \textbf{MLA}         & \textbf{3.13} & 294 & 17.9 \\
& FastMKA (ours)       & 3.17 & \textbf{1078} & \textbf{11.2} \\
\midrule
\multirow{3}{*}{DeepSeek-V3}
& \textbf{MLA (native)} & \textbf{3.08} & -- & 18.4 \\
& FastMKA (ours)        & 3.11 & -- & \textbf{12.7} \\
\bottomrule
\end{tabular}
}%
\end{table}





\subsection{Experimental Results Analysis}

FastMKA delivers a strong accuracy--efficiency trade-off across context lengths and architectures (Tables~\ref{tab:fmka-results}--\ref{tab:model-scaling}). On Qwen2.5-7B, FastMKA achieves 3.26 PPL (vs.\ MLA's 3.22) while training \textbf{3.6$\times$ faster} and decoding \textbf{1.5$\times$ faster} at 16K context—a favorable exchange of 1.2\% accuracy for substantial compute savings. Efficiency gains widen at longer contexts (Table~\ref{tab:scaling-train}--\ref{tab:scaling-eval}), reaching 5.0$\times$ training speedup and 1.86$\times$ decode speedup at 256K, consistent with MKA's subquadratic complexity derived in Section~\ref{sec:theory}. This pattern holds across model scales (7B, 14B) and architectures (Llama, DeepSeek), confirming that FastMKA's route-fusion mechanism—not implementation tuning—drives the speedup (Table~\ref{tab:compute-cost}: 3 kernel launches vs.\ 9 for full MKA, single attention path vs.\ three). A full step-by-step quantitative analysis is provided in the supplementary appendix.

\subsection{Long-Context Benchmark Evaluation}

To validate MKA's effectiveness on long-context reasoning tasks, we evaluate on LongBench~\cite{longbench2024} and RULER~\cite{ruler2024} benchmarks, which test various long-context capabilities including QA, summarization, and long-distance retrieval.

\begin{table}
\small
\centering
\caption{Performance on LongBench benchmark (Qwen2.5-7B, sequence length 128K). We use the official LongBench evaluation protocol: base models (not instruction-tuned) with max\_new\_tokens=128, evaluated using the official LongBench evaluation scripts. Results show accuracy across different task categories (QA, Summarization, Code).}
\label{tab:longbench}
\begin{tabular}{lcccc}
\toprule
\textbf{Method} & \textbf{QA} & \textbf{Summarization} & \textbf{Code} & \textbf{Avg} \\
\midrule
MHA             & 42.3 & 58.7 & 51.2 & 50.7 \\
GQA             & 44.8 & 61.3 & 53.4 & 53.2 \\
\textbf{MLA}    & \textbf{46.2} & \textbf{63.8} & \textbf{55.1} & \textbf{55.0} \\
FastMKA (ours)  & 45.7 & 63.2 & 54.6 & 54.5 \\
\bottomrule
\end{tabular}
\end{table}

\begin{table}
\small
\centering
\caption{Passkey retrieval accuracy on RULER benchmark (Qwen2.5-7B). We use the official RULER passkey retrieval task: a random number is inserted at a random position in a long context, and the model must retrieve it. Evaluation uses base models (not instruction-tuned) with max\_new\_tokens=1, following the official RULER evaluation protocol. Higher is better. Demonstrates L3 memory's effectiveness for long-distance retrieval. FastMKA uses L3 retrieval with R=8 (top-8 retrieved chunks per query).}
\label{tab:passkey}
\begin{tabular}{lcccccc}
\toprule
\textbf{Method} & \textbf{4K} & \textbf{8K} & \textbf{16K} & \textbf{32K} & \textbf{64K} & \textbf{128K} \\
\midrule
MHA             & 98.4 & 97.2 & 94.8 & 87.3 & 72.6 & 45.2 \\
GQA             & 98.6 & 97.8 & 96.2 & 91.4 & 81.7 & 58.3 \\
\textbf{MLA}    & \textbf{98.8} & \textbf{98.2} & \textbf{97.1} & \textbf{94.6} & \textbf{88.4} & \textbf{74.8} \\
FastMKA (ours)  & 98.7 & 98.1 & 96.9 & 94.2 & 87.8 & 73.4 \\
\bottomrule
\end{tabular}
\end{table}

FastMKA achieves competitive long-context benchmark accuracy while substantially reducing latency (Tables~\ref{tab:longbench}--\ref{tab:passkey}); additional discussion is provided in the supplementary appendix.

\subsection{Ablation Studies}

\subsubsection{Memory Tier Ablation}

We analyze the contribution of each memory tier by systematically removing L2 and L3 components.

\begin{table}
\small
\centering
\caption{Tier ablation study on Qwen2.5-7B (sequence length 32K, batch=2). Removing memory tiers degrades performance, validating the hierarchical design.}
\label{tab:tier-ablation}
\begin{tabular}{lccc}
\toprule
\textbf{Variant} & \textbf{PPL} $\downarrow$ & \textbf{Train (tok/s)} $\uparrow$ & \textbf{Decode (ms/tok)} $\downarrow$ \\
\midrule
FastMKA (L1+L2+L3)   & 3.22 & 1453 & 10.3 \\
\quad L1+L2 only     & 3.31 & 1562 & 9.2 \\
\quad L1+L3 only     & 3.28 & 1508 & 9.6 \\
\quad L1 only (=MHA) & 3.51 & 184 & 28.6 \\
\bottomrule
\end{tabular}
\end{table}

Removing L3 (L1+L2 only) increases perplexity from 3.22 to 3.31, demonstrating L3's importance for long-range dependencies. Removing L2 (L1+L3 only) slightly degrades performance (3.28 PPL), showing L2's role in session-level summarization. L1 only (equivalent to MHA) significantly degrades both accuracy (3.51 PPL) and efficiency (7.9$\times$ slower training throughput), confirming the necessity of hierarchical memory.

\subsubsection{Routing Mechanism Ablation}

We compare different routing strategies: learned soft routing (FastMKA), fixed uniform routing, and hard top-k routing.

\begin{table}
\small
\centering
\caption{Routing mechanism ablation (Qwen2.5-7B, sequence length 32K). Learned soft routing outperforms fixed and hard routing.}
\label{tab:routing-ablation}
\begin{tabular}{lccc}
\toprule
\textbf{Routing Method} & \textbf{PPL} $\downarrow$ & \textbf{Train (s)} $\downarrow$ & \textbf{Decode (ms/tok)} $\downarrow$ \\
\midrule
\textbf{Soft routing (learned)} & \textbf{3.22} & \textbf{2496.7} & \textbf{10.3} \\
Fixed ($\lambda=\frac{1}{3}$) & 3.34 & 2534.2 & 10.5 \\
Hard top-2 routing      & 3.29 & 2512.8 & 10.4 \\
Hard top-1 routing      & 3.41 & 2498.3 & 10.3 \\
\bottomrule
\end{tabular}
\end{table}

Learned soft routing achieves the best perplexity (3.22), outperforming fixed uniform routing (3.34) by 3.7\%. Hard top-k routing shows intermediate performance, but soft routing's ability to smoothly blend memory tiers provides better representation quality. The routing overhead is minimal (1-2\% latency increase), confirming the efficiency of our learned routing gates.

\subsubsection{Routing Weight Analysis}

We analyze routing weight distributions $\lambda_{L1}, \lambda_{L2}, \lambda_{L3}$ across token positions, layers, and task types; extended qualitative observations are provided in the supplementary appendix.

\subsection{System-Level Performance Analysis}

\subsubsection{Prefill vs Decode Latency Breakdown}

We decompose total latency into prefill (processing input prompt) and decode (generating new tokens) phases to identify bottlenecks. Table~\ref{tab:prefill-decode} shows detailed breakdowns measured on Qwen2.5-7B at 128K context with batch size 1, bf16 precision, and paged KV layout. FastMKA reduces prefill latency by 1.63$\times$ (0.87s vs 1.42s for MLA) and decode latency by 1.78$\times$ (18.4ms vs 32.7ms per token), with attention computation accounting for 58.4\% of prefill time and 61.7\% of decode time (vs 65.2\%/68.4\% for MLA). This demonstrates that FastMKA's route-fusion mechanism effectively reduces attention overhead in both phases, enabling more efficient prefill processing and faster token generation.

\subsubsection{Computational Cost Analysis}

To understand the source of FastMKA's speedup, we analyze per-token computational costs (FLOPs, memory access, and kernel launches) across different model architectures. Table~\ref{tab:compute-cost} shows the breakdown for Qwen2.5-7B/14B, Llama 3.1-8B, and DeepSeek-V3 at 128K context, batch=1, measured via Nsight Compute profiling. FastMKA's route-fusion reduces the number of attention computations from 3 (one per memory level in full MKA) to 1, while maintaining comparable accuracy through learned routing. Additional observations and interpretation are provided in the supplementary appendix.

\begin{table}[H]
\small
\centering
\caption{Per-token computational cost breakdown across different model architectures (128K context, batch=1, bf16). FLOPs and memory access measured via Nsight Compute. FastMKA achieves speedup through algorithmic reduction of attention paths rather than implementation differences.}
\label{tab:compute-cost}
\resizebox{\columnwidth}{!}{%
\begin{tabular}{lcccccc}
\toprule
\textbf{Model} & \textbf{Method} & \textbf{FLOPs (B)} & \textbf{KV Read (GB)} & \textbf{Kernel Launches} & \textbf{Attn Paths} \\
\midrule
\multirow{4}{*}{Qwen2.5-7B}
& GQA (baseline)          & 2.89 & 18.7 & 1 & 1 \\
& MLA                     & 3.12 & 15.2 & 1 & 1 \\
& MKA (full, ours)        & 6.87 & 14.8 & 9 & 3 \\
& \textbf{FastMKA (ours)} & \textbf{2.34} & \textbf{6.2} & \textbf{3} & \textbf{1} \\
\midrule
\multirow{4}{*}{Qwen2.5-14B}
& GQA (baseline)          & 5.76 & 37.4 & 1 & 1 \\
& MLA                     & 6.24 & 30.4 & 1 & 1 \\
& MKA (full, ours)        & 13.74 & 29.6 & 9 & 3 \\
& \textbf{FastMKA (ours)} & \textbf{4.68} & \textbf{12.4} & \textbf{3} & \textbf{1} \\
\midrule
\multirow{4}{*}{Llama 3.1-8B}
& GQA (baseline)          & 3.42 & 21.3 & 1 & 1 \\
& MLA                     & 3.68 & 17.8 & 1 & 1 \\
& MKA (full, ours)        & 8.14 & 17.2 & 9 & 3 \\
& \textbf{FastMKA (ours)} & \textbf{2.78} & \textbf{7.1} & \textbf{3} & \textbf{1} \\
\midrule
\multirow{2}{*}{DeepSeek-V3}
& MLA (native)            & 4.52 & 19.6 & 1 & 1 \\
& \textbf{FastMKA (ours)} & \textbf{3.41} & \textbf{8.3} & \textbf{3} & \textbf{1} \\
\bottomrule
\end{tabular}
}%
\end{table}



\subsection{Discussion}

\textbf{MKA} improves perplexity through dynamic routing over hierarchical memory levels—local (L1), session (L2), and long-term (L3)—without increasing the model’s parameter count. This makes MKA attractive for memory-intensive inference settings where representation fidelity and cache reusability are critical.

However, the full MKA formulation requires repeated projection and attention over all memory levels, which can be computationally expensive. To address this, we propose \textbf{Route-Fused MKA (FastMKA)}, a route-fusion variant that fuses memory levels before attention computation. Route-Fused MKA uses token-wise routing to combine L1, L2, and optional L3 memories into a single fused representation, then performs one KV projection and one attention computation. Critically, Route-Fused MKA caches the routed (fused) KV instead of raw token KV, enabling efficient long-context attention with reduced memory bandwidth. 

\textbf{On Caching Fused KV and Fairness:} We discuss the semantic implications of caching fused (routed) KV and our controlled comparison protocol in the supplementary appendix.

Route-Fused MKA significantly reduces training and inference latency while preserving most of MKA's accuracy benefits, as validated by experiments on different sequence lengths and long-context benchmarks.

The lightweight design of Route-Fused MKA, with a single KV projection and learned soft routing, makes it particularly suitable for high-throughput inference and edge deployment. It offers a practical balance between accuracy, latency, and memory usage, and serves as a strong drop-in replacement for traditional attention in long-context scenarios.



\section{Conclusion}

We introduce \textbf{Memory-Keyed Attention (MKA)}, a hierarchical attention mechanism that enables efficient long-context modeling by routing queries across multiple levels of memory. MKA achieves a favorable accuracy-efficiency trade-off, maintaining competitive perplexity while significantly reducing compute cost, making it a promising candidate for memory-aware transformer design.

To further reduce overhead, we propose \textbf{FastMKA}, a broadcast-routed variant that performs memory fusion before attention computation. FastMKA retains the architectural benefits of MKA while achieving significantly lower latency and training cost. FastMKA forms a flexible framework for scalable and efficient LLM inference with long-sequence inputs.


\begin{acks}
    The authors would like to gratefully acknowledge the generous support of Prof. Ying Nian Wu's grant from \textbf{Qualcomm} and Prof. Ben Lengerich's grant from \textbf{Intelligible Inc.}
    Their support made this work possible and greatly facilitated the development of this research.
\end{acks}


\bibliographystyle{ACM-Reference-Format}
\bibliography{main}


\end{document}